\documentclass{article}

\usepackage[dblblindworkshop, final]{neurips_2026}

\usepackage[utf8]{inputenc} 
\usepackage[T1]{fontenc}    
\usepackage[pagebackref]{hyperref}
\usepackage{url}            
\usepackage{booktabs}       
\usepackage{amsfonts}       
\usepackage{nicefrac}       
\usepackage{microtype}      
\usepackage{graphicx}
\usepackage{xcolor}         

\hypersetup{
    colorlinks,
    linkcolor={red!50!black},
    citecolor={blue!50!black},
    urlcolor={blue!50!black}
}

\usepackage{subcaption}
\usepackage{multicol}
\usepackage{multirow}

\usepackage{amsmath}
\usepackage{pifont}
\newcommand{\cmark}{\ding{51}}%
\newcommand{\xmark}{\ding{55}}%

\usepackage[capitalize,noabbrev]{cleveref}

\renewcommand*{\backrefalt}[4]{%
   \ifcase #1 \or \footnotesize(Cited on p.~#2)\else \footnotesize(Cited on pp.~#2)\fi}

\crefformat{chapter}{\S#2#1#3}
\crefformat{section}{\S#2#1#3}
\crefformat{appendix}{\S#2#1#3}
\crefformat{subsection}{\S#2#1#3}
\crefformat{subsubsection}{\S#2#1#3}
\crefrangeformat{section}{\S\S#3#1#4 to~#5#2#6}
\crefmultiformat{section}{\S\S#2#1#3}{ and~#2#1#3}{, #2#1#3}{ and~#2#1#3}

\usepackage{listings}

\definecolor{codegreen}{rgb}{0,0.6,0}
\definecolor{codegray}{rgb}{0.5,0.5,0.5}
\definecolor{codepurple}{rgb}{0.58,0,0.82}
\definecolor{backcolour}{rgb}{0.95,0.95,0.92}

\lstdefinestyle{mystyle}{
    backgroundcolor=\color{backcolour},   
    commentstyle=\color{codegreen},
    keywordstyle=\color{magenta},
    numberstyle=\tiny\color{codegray},
    stringstyle=\color{codepurple},
    basicstyle=\ttfamily\footnotesize,
    breakatwhitespace=false,         
    breaklines=true,                 
    captionpos=b,                    
    keepspaces=true,                 
    numbers=left,                    
    numbersep=5pt,                  
    showspaces=false,                
    showstringspaces=false,
    showtabs=false,                  
    tabsize=2
}

\title{Simplifying, Scaling, and Saturating\\ Flow Matching for Cell Microscopy}
\workshoptitle{ML4SpatialBio}

\author{%
  Charles Jones\thanks{Correspondence to: \texttt{charlie@valencelabs.com}} \\
  Valence Labs\\
  \And
  Emmanuel Noutahi \\
  Valence Labs\\
  \And
  Zeinab Navidi\thanks{Work done at Valence Labs} \\
  Vector Institute\\
  \And
  Alessandro Palma\footnotemark[2]\\
  Helmholtz Munich \\
  \AND
  Kian Kenyon-Dean\\
  Recursion \\
  \And
  Cian Eastwood\footnotemark[2]\space\space\thanks{Equal advising} \\
  Relation Therapeutics\\
  \And
  Jason Hartford\footnotemark[3] \\
  Valence Labs\\
}
\begin{document}

\maketitle

\begin{abstract}
  Flow-matching generative models are increasingly used to simulate cell responses to biological perturbations.
  However, the design space for building such models is large and underexplored. We systematically analyse the design space of flow matching models for cell microscopy images, finding that many popular techniques can add unnecessary complexity and may even hurt performance. We develop a simple, stable, and scalable recipe which we use to train our generative microscopy model. We scale training compute to an order of magnitude larger than prior methods, saturating existing public benchmarks with approximately two-fold FID and ten-fold KID improvements over prior methods. We report an extended battery of metrics, finding that improved evaluation encoders can alleviate the saturation problem. Our methods represent progress towards high-fidelity virtual microscopy assays for tasks such as virtual screening, batch effect correction, and data enrichment.
\end{abstract}

\section{Introduction}

High-throughput screens with high-content readouts have become indispensable tools in modern biology, providing large datasets that capture high-dimensional cellular states across diverse conditions \citep{ljosa2012annotated,sypetkowski2023rxrx1,fay2023rxrx3}. However, despite this wealth of data, exhaustive physical experimentation remains expensive due to the combinatorial complexity of the search space (e.g. exploring cell perturbations under all contexts, all gene knockout combinations, or all of chemical space). Consequently, we have seen increasing efforts to build `virtual cell' models~\citep{slepchenko2003quantitative,bunne2024build,noutahi2025virtual} to simulate cellular responses to perturbations \textit{in silico}. Successfully achieving this vision would greatly impact science and drug discovery by enabling applications such as virtual screening of potential therapies.

While much work has focused on transcriptomics data, there is a growing body of work focused on image-based profiling \citep{bray2016cell}, which offers a cost-effective and spatially resolved view of cellular phenotypes \citep{navidi2025morphodiff,palma2025predicting,zhang2025cellflux}. In particular, following notable successes in natural image generation and editing tasks \citep{dhariwal2021diffusion}, diffusion and flow matching generative models have emerged as tools for generating high fidelity images of cellular morphology under biological perturbations \citep{bourou2024phendiff,zhang2025cellflux}. However, despite some promising initial results, the design space for these models remains underexplored, with little consensus on best training practices such as how to model batch effects.

Inspired by recent work in natural imaging \citep{karras2022elucidating,karras2024analyzing,ma2024sit}, we systematically analyse the design space of flow matching generative models for cellular microscopy. We find that prior models are likely to be undertrained and that many intuitive methods (e.g. control-to-perturbed flows, optimal transport couplings)  add unnecessary complexity and can even degrade performance. By rectifying this, we show that it is possible to achieve large improvements over prior state-of-the-art methods on the RxRx1, BBBC021, and JUMP benchmarks.

\begin{figure}[ht!]
    \centering
    
    \includegraphics[width=0.95\linewidth]{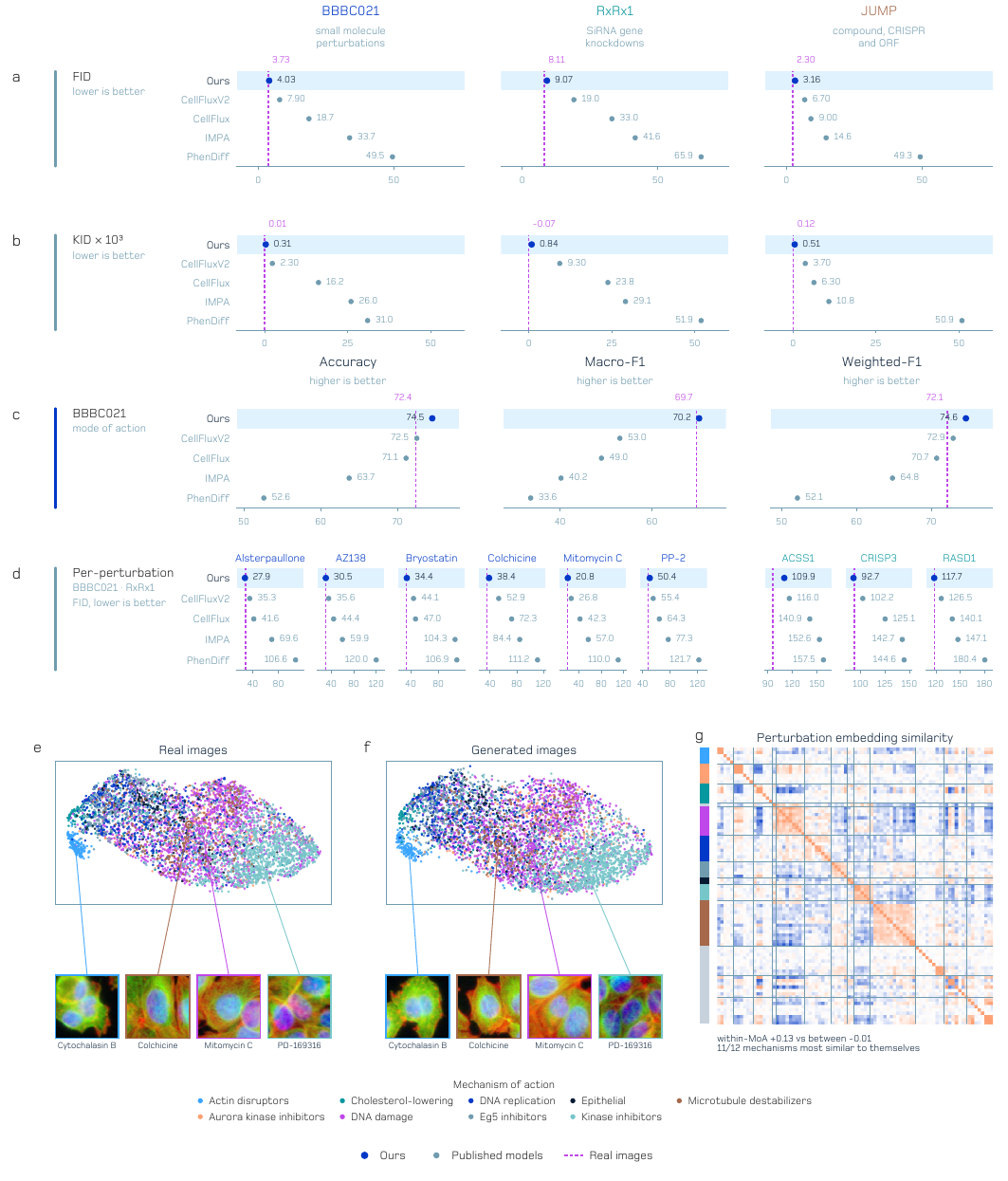}
    \caption{\textbf{We ablate, simplify, and scale flow matching for generative modelling of cellular microscopy, surpassing state-of-the-art performance and saturating the RxRx1, BBBC021, and JUMP benchmarks.} We measure the performance of our best configuration on overall FID (\textbf{a}), KID $\times 10^3$  (\textbf{b}), mechanism of action (MoA) classification (\textbf{c}), and per-perturbation FID (\textbf{d}). Magenta rules mark FID and KID between the evaluation set and real images from the training set, representing an approximate upper bound for model performance. We plot UMAP projections of DINOv2 features for real (\textbf{e}) and generated (\textbf{f}) images on BBBC021, coloured by mechanism of action of the perturbation. We plot cosine similarities between the model's learned perturbation conditioning embeddings (\textbf{g}), ordered by mechanism of action. The model was trained on one-hot perturbation labels and never saw mechanism annotations, so the block structure on the diagonal is inferred from morphology alone.}  
    \label{fig:main-fig}
\end{figure}

\paragraph{This paper} We provide an overview of our main results in \cref{sec:saturating} and discuss the key takeaways of our work in \cref{sec:discussion}. To make our recipe for simple, scalable virtual microscopy assays easy to build on, we provide code at \href{https://github.com/valence-labs/microscopy-flow-matching}{github.com/valence-labs/microscopy-flow-matching}; we further provide full details of our ablations, methods, and results in the Appendix, which we structure in the style of an extended technical report. Specifically, we provide details on:

\begin{itemize}
    \item[\cref{sec:iid}] Ablations over the design space of flow matching models for generative microscopy, including conditioning, flow construction, interpolants, data coupling, and architecture.
    \item[\cref{sec:simulation}] Results applying our models to the BBBC021 unseen molecules virtual screening task. We test three different strategies for adapting our base models to unseen perturbations.
    \item[\cref{sec:datadetails}] Datasets and evaluation metrics used in the study, as well as full training recipes and extended results for each model configuration tested throughout \cref{sec:iid,sec:simulation}.
    \item[\cref{sec:qualitativeresults}] Additional qualitative comparisons of real and generated images for all tasks, including RxRx1, JUMP, BBBC021 (seen molecules), and BBBC021 (unseen molecules).
\end{itemize}

\section{Saturating open benchmarks in generative microscopy} \label{sec:saturating}

Through our ablations in \cref{sec:iid}, we found, somewhat surprisingly, that we could remove many domain-specific methodologies for generative microscopy. Our best results came from a conventional noise-to-data pixel-space flow matching methodology, using a diffusion transformer architecture with minor tweaks for stability on microscopy data. In \cref{sec:iid}, we additionally demonstrate how to sample from the pretrained model in \textit{counterfactual} mode, enabling us to use control images at inference time to predict changes to specific cells. In \cref{fig:main-fig}\{a,b\}, we show how our final training recipe improves over prior state-of-the-art \citep[CellFluxV2;][]{Zhangcellfluxv2} by roughly a factor of two on FID and ten on KID across all benchmarks. We similarly demonstrate state-of-the-art results on perturbation-specific metrics such as mechanism of action prediction (\cref{fig:main-fig}c) and conditional FID (\cref{fig:main-fig}d). Furthermore, our final model demonstrates qualitative alignment between real and generated images in DINOv2 space (\cref{fig:main-fig}\{e,f\}), as well as capability to implicitly learn mechanism of action relationships between different perturbations (\cref{fig:main-fig}g).

\begin{figure}[ht]
    \centering
    \includegraphics[width=\linewidth]{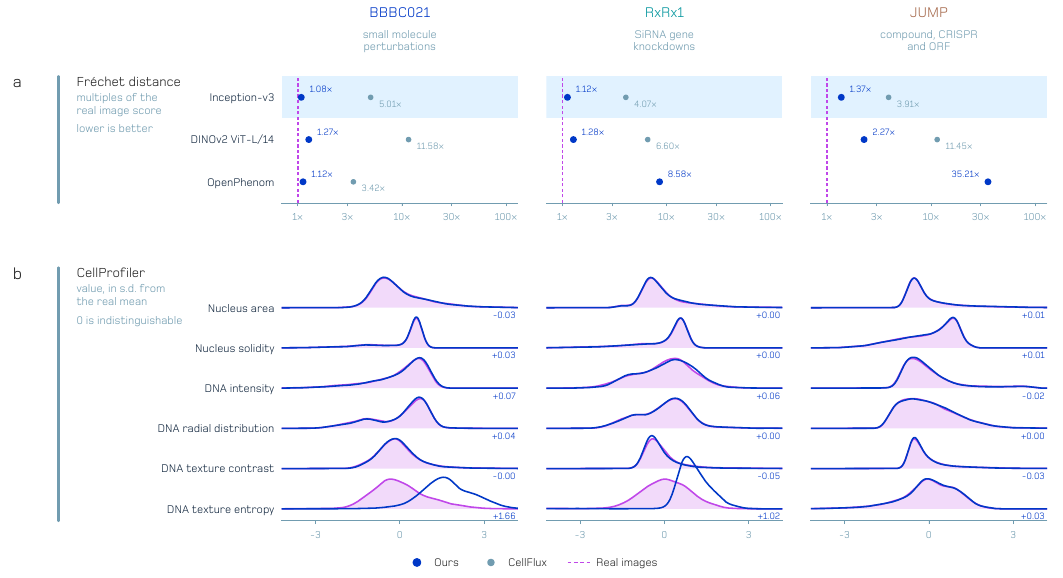}
    \caption{\textbf{Metric saturation depends on the feature extractor.} We recompute Fr\'echet distances as a multiple of the real image score for our method and CellFlux \citep{zhang2025cellflux} using DINOv2 and OpenPhenom as alternative evaluation encoders (\textbf{a}). We find that DINOv2 gives better separation between both models and both encoders give more headroom to the real image score. Inception and DINOv2 operate in RGB space, whereas OpenPhenom and CellProfiler use the native channels. Since CellFlux released RGB projections only, it cannot be scored with OpenPhenom on RxRx1 and JUMP. We plot distributions of six normalised CellProfiler nucleus features (\textbf{b}), finding that modelling error is concentrated in texture, while nucleus size and shape are indistinguishable from real images.}  
    \label{fig:encoder-fig}
\end{figure}

Following recent work in natural imaging \citep{yang2026representation}, we obtain real image baselines by computing FID and KID between the evaluation set and a set of real images drawn from the training set (with matched number of images and composition of perturbations and experiments to the evaluation set). While these are \textit{not} hard upper bounds on performance, they approximate the expected score of a perfect generative model. Across all three benchmarks, our model reaches within one FID point of the real image score, indicating that FID is beginning to saturate. We further investigate this result in \cref{fig:encoder-fig}a, finding that evaluating models with improved feature extractors (OpenPhenom, DINOv2) can help to alleviate the saturation problem by providing more headroom to the real image score.

Finally, while metrics like FID and KID measure how far the generated distribution is to the real one, they do not specify the biological strengths and weaknesses of the models. We thus additionally report distributions of CellProfiler \citep{carpenter2006cellprofiler} scores across six nucleus features in \cref{fig:encoder-fig}b. We find that our model is biologically faithful at reproducing geometric structure such as nucleus size, shape, and radial distribution. However, our model is weaker on Haralick texture on both BBBC021 and RxRx1. It remains unclear whether this is due to fundamental limitations of the model class to resolve fine-grained textural details, or simply image quantisation artefacts.

\section{Discussion} \label{sec:discussion}

\paragraph{High fidelity virtual microscopy assays are simpler than you think} A central challenge in machine learning for high-dimensional and spatial biology is how (and if) we should integrate domain priors into our models. Existing flow matching methods for generative microscopy \citep[e.g.][]{zhang2025cellflux,Zhangcellfluxv2} diverge from standard generative modelling methods due to domain-specific assumptions about how to construct the flow and condition on biological context, such as cell types and batch effects (see \cref{sec:addressing_misconceptions} for a more involved discussion). Through the largest-scale ablations to date, we found that while these approaches are intuitively appealing, they provided no benefit over simpler methods that are standard in the generative modelling community. By providing the community with a simpler, more scalable recipe, we hope to leave researchers in a stronger position to build models that can deliver capabilities such as virtual screening and batch effect correction.



\paragraph{Molecular representations are key to virtual screening} One downstream goal of generative microscopy is to realise the vision of virtual screening \citep{carpenter2007image} by simulating effects of novel candidate drugs on cells (unseen at train time). We can improve common metrics associated with virtual screening by improving the generative model itself: in \cref{sec:simulation}, we showed that an \textit{unconditional} variant of our model beats almost all prior methods on the BBBC021 unseen compound split simply due to the fact that it generates more realistic images. However, to achieve truly interesting capabilities, we need to condition the model on molecule representations. By conditioning our model on MolGPS \citep{sypetkowski2024scalability}, a 3B parameter pretrained graph transformer, we achieve state-of-the-art FID for effects of unseen compounds, improving over the prior method of Morgan fingerprint embeddings. However, the gap was small. As methods for molecular representation learning improve, we expect to see virtual screening applications become more feasible.

\paragraph{Existing benchmarks are beginning to saturate} For consistency and comparison with prior work, we report all results on the standard, open, RxRx1, BBBC021, and JUMP benchmarks. In \cref{fig:main-fig}, we demonstrate that conventional metrics (FID, KID, MoA) are becoming saturated. Future work may wish to begin reporting metrics using modern or domain-specific evaluation encoders such as DINOv2 or OpenPhenom, which we demonstrate provide better separation between different models (\cref{fig:encoder-fig}). More broadly, these benchmarks are small in number of images, phenotypic diversity (i.e. cell types and biological batches), and number of perturbations. Furthermore, the standard crop of $96\times96$px, centred on a single nucleus is likely too small to capture effects of perturbations across multiple scales (such as changes to cell arrangement and inter-cellular relations or viability/toxicity). While these datasets remain useful for prototyping, we encourage researchers to consider the openly available RxRx3 and RxRx3-core \citep{fay2023rxrx3,kraus2025rxrxcore} datasets as larger, more biologically diverse alternatives for future benchmarking.


\bibliography{refs}
\bibliographystyle{icml2026}

\newpage
\appendix

\section{Ablated generative microscopy} \label{sec:iid}

We use the publicly available RxRx1 \citep{sypetkowski2023rxrx1} dataset, evaluating Fréchet Inception Distance (FID) and Kernel Inception Distance (KID) on the held-out validation set from \citet{zhang2025cellflux}. Our experiments aim to explore and simplify five elements of the flow matching design space: conditioning, flow construction, interpolants, data coupling, and architecture.

\begin{table*}[ht]
    \centering
    \caption{\textbf{Systematically improving upon state-of-the-art microscopy generative modelling on RxRx1.} Our final configuration achieves over two-fold FID improvements and ten-fold KID improvements over prior and contemporaneous work. ExaFLOPs measure the total training compute, as estimated in \cref{sec:flops_estimation}. KID numbers are scaled up by a factor of $10^3$ for ease of interpretation.}
    \begin{tabular}{@{}lcccccc@{}}
    \toprule
        & \textbf{Model} & \textbf{ExaFLOPs} & \textbf{Flow} & \textbf{OT} & \textbf{FID} $\downarrow$ & \textbf{KID} $\downarrow$ \\ \midrule
        PhenDiff \citep{bourou2024phendiff} & U-Net & 0.96 & $\mathcal{N \leftrightarrow D}$ & \xmark & 65.9 & 51.9 \\
        
        IMPA \citep{palma2025predicting} & StarGAN & 0.39 & - & -  & 41.6 & 29.1 \\ 
        
        CellFlux \citep{zhang2025cellflux} & ADM & 7.83 & $\mathcal{C \rightarrow P}$ & \xmark  & 33.0 & 23.8 \\ 

        CellFluxV2 \citep{Zhangcellfluxv2} & DiT-XL/2 & 6.89 & $\mathcal{C \rightarrow P}$ & \xmark  & 19.0 & 9.30 \\ \midrule
        
        Ours -- \textit{Config A} & ADM & 4.70 & $\mathcal{C \rightarrow P}$ & \xmark & 29.5 & 14.0 \\

        Ours -- \textit{Config B (undertrained)} & ADM & 3.93 & $\mathcal{N \rightarrow D}$ & \xmark & 12.3 & 3.90 \\
        
        Ours -- \textit{Config B} & ADM & 15.7 & $\mathcal{N \rightarrow D}$ & \xmark & 9.90 & 1.52 \\
        
        Ours -- \textit{Config C} & ADM & 15.7 & $\mathcal{N \leftrightarrow D}$ & \xmark & 13.4 & 1.14 \\
        
        Ours -- \textit{Config D} & ADM & 15.7 & $\mathcal{N \rightarrow  D}$ & \cmark & 10.5 & 1.61 \\

         Ours -- \textit{Config E} & ADM & 15.7 & $\mathcal{N \leftrightarrow  D}$ & \cmark & 13.6 & 2.15\\
        
        \midrule
        
        \textit{Ours (best configuration)} & MiT-XL/2 & 118 & $\mathcal{N \rightarrow D}$ & \xmark & \textbf{9.07} & 0.84 \\
        
        \quad \textit{w/ counterfactual sampling} & MiT-XL/2 & 118 & $\mathcal{N \leftrightarrow D}$ & \xmark & 10.4 & \textbf{0.64}  \\        
    \bottomrule
    \end{tabular}
    \label{tab:rxrx1_perf}
\end{table*}


\subsection{Conditioning} 

The prior state-of-the-art method for microscopy generative modelling \citep{zhang2025cellflux} uses the classic ADM U-Net architecture from \citet{dhariwal2021diffusion}, conditioned on gene2vec \citep{du2019gene2vec} embeddings of the siRNA gene knockdowns in RxRx1. We begin by simplifying this conditioning setup by using one-hot embeddings of the perturbation label. While this strategy loses the ability to generalise to unseen perturbations, it is simpler and arguably more suitable for the seen perturbation setting\footnote{Also note that, since the number of human genes is finite and relatively small (approximately 20k), it's feasible to run whole-genome knockout experiments \citep{fay2023rxrx3}. Thus, unlike compounds, which inhabit a much larger chemical space, there is little need to generalise to unseen gene perturbations on RxRx1 or similar datasets.} (we will revisit perturbation encoders for unseen molecule generalisation in \cref{sec:simulation}). Importantly, we additionally condition on a one-hot embedding of the biological context (experiment label) to allow our models to capture the batch effects known in RxRx1 and other microscopy datasets \citep{sypetkowski2023rxrx1}. Our full conditioning strategy is visualised in \cref{fig:arch_cond_strategy}. At train time, we independently drop each label with $0.15$ probability, allowing us the option to use classifier-free guidance \citep{ho2022classifier} and to sample from our models unconditionally. We report results for our simplified conditioning setup in \textit{Config A} -- despite closely matching other aspects of training to \citep{zhang2025cellflux}, we see a modest improvement in FID to $29.5$.

One additional advantage of conditioning on learnable perturbation embeddings is that we can analyse and inspect the embeddings themselves. In \cref{fig:main-fig}g, we demonstrate how we can use pairwise cosine similarity between the embeddings learned by our final model to recover mechanism of action relationships between the perturbations. This is a new capability -- the model did not see mechanism of action annotations at train time.

\begin{figure}[ht]
    \centering
    \includegraphics[width=\linewidth]{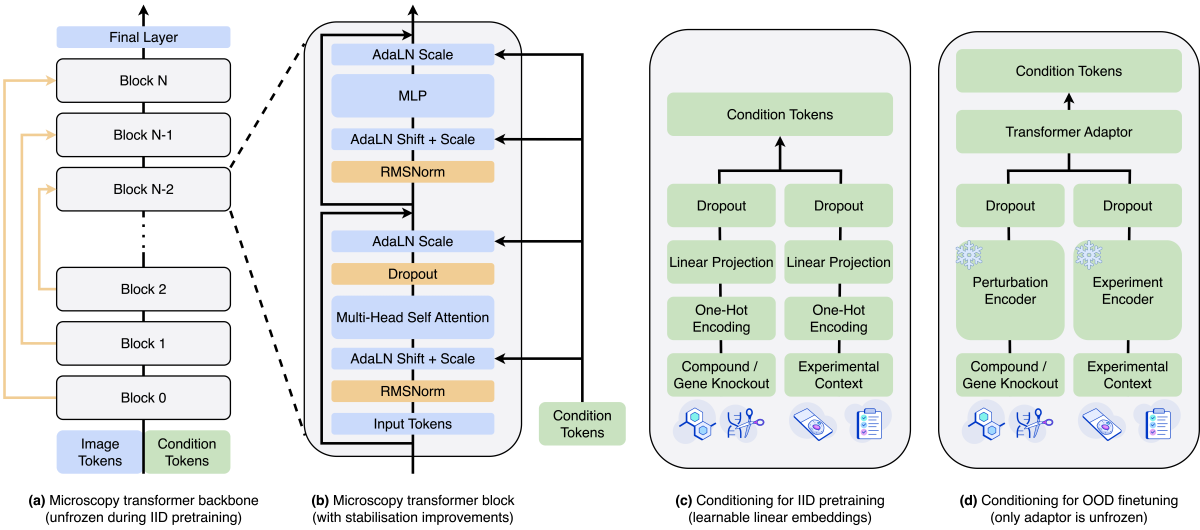}
    \caption{\textbf{Architecture and conditioning.} We make three adjustments to the standard diffusion transformer, including long-range skip connections, RMSNorm, and dropout (\textbf{a},\textbf{b}). While (pre)training on seen perturbations (\textbf{c}), we use learnable linear embeddings of the perturbation label. When finetuning for unseen perturbations in \cref{sec:simulation}, we freeze the base model and perturbation encoder. We train a small adaptor to align the embeddings to the model's conditioning (\textbf{d}).}
    \label{fig:arch_cond_strategy}
\end{figure}

\subsection{Flow construction} 

An important choice in flow matching generative models is the interpolant between source and target distributions. Standard practice in natural images is to choose a Gaussian source distribution, however, in microscopy data, we have access to control images -- cells which are left untreated or given a non-coding gene perturbation -- and it is intuitive to wish to use these as the source distribution.

In \textit{Config A}, we match the setup of \citet{zhang2025cellflux} by flowing from the control distribution to the perturbed distribution $\mathcal{C \rightarrow P}$, conditioned on the biological experiment. However, one undesirable property of this setup is that we cannot simulate a perturbation without a corresponding control image. In practice, we found the $\mathcal{C \rightarrow P}$ flow to be prone to overfitting and challenging to scale. 

In \textit{Configs} \{\textit{B,C}\}, we explore two alternative sampling strategies based on the more flexible approach of constructing a Gaussian noise to all-data flow $\mathcal{N \rightarrow D}$, where controls are simply treated as another perturbation class. Training with this flow enables us to use either of the following sampling strategies (summarised in \cref{fig:rxrx1_interpolants}):

\paragraph{Generation mode} ($\mathcal{N \rightarrow D}$) is akin to standard diffusion or flow matching and involves sampling $x_0 \sim \mathcal{N}(0,\mathbf{I})$, then solving the (neural) ODE $x_1 = x_0 + \int^1_0 v_\theta(x_t,t,c)\mathrm{d}t$, where $v_\theta(\dots)$ is the output of the model, mapping an image, conditioning, and time scalar to a predicted velocity field. Common practice is to approximate this integral with the Euler ODE solver, corresponding to DDIM diffusion \citep{song2021denoising}, however, any ODE or SDE solver may be used in-place with no changes required at train-time.

\paragraph{Counterfactual mode} ($\mathcal{N \leftrightarrow D}$) involves starting with a control image $x_c$, solving the reverse probability flow ODE for the corresponding noise, then decoding the noise forwards with new conditioning to a perturbed image ${x_1 = x_c +\int^0_1 v_\theta(x_t,t,\emptyset)\mathrm{d}t + \int^1_0 v_\theta(x_t,t,c)\mathrm{d}t}$. Unlike generation mode, which simply generates a cell with the given condition, counterfactual mode aims to answer the more challenging question of `what would \textit{this} cell look like under a given condition'. This is the strategy taken by \citet{bourou2024phendiff} -- we refer to it as \textit{counterfactual} generation as it closely resembles \textit{abduction-action-prediction} framework of \citet{pearl2009causality}\footnote{While we use the \textit{counterfactual} terminology to highlight the connections to causality, we do not make any theoretical causal claims about our capabilities. \citet{ribeiro2025counterfactual} provide a full treatment of the causal identifiability of flow matching models.}, common in image-editing and causality \citep{pawlowski2020deep,ribeiro2025counterfactual}.

\begin{figure}[ht]
    \centering
    \includegraphics[width=\linewidth]{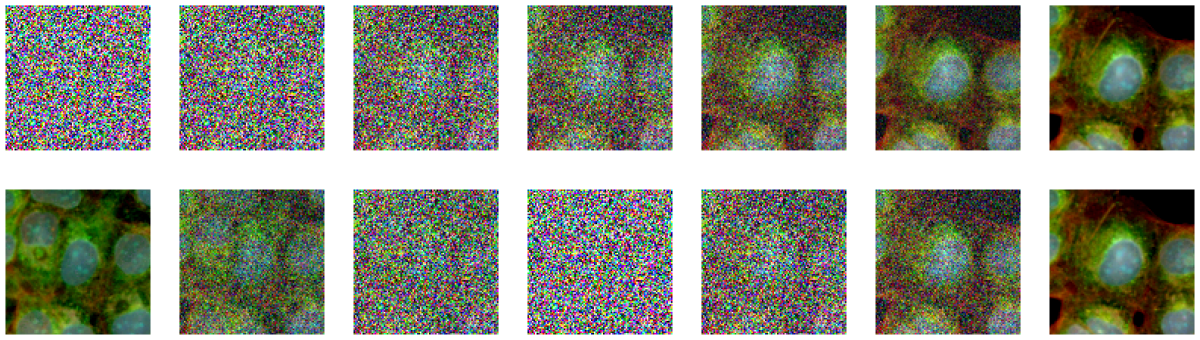}
    \caption{\textbf{Generation vs. counterfactual mode sampling.} Top row represents the flow trajectory for generation mode, starting with Gaussian noise and ending at a perturbed image. Bottom row represents the trajectory for counterfactual mode. This involves starting with a control image, solving the reverse probability flow ODE to compute its corresponding noise, then generating a perturbed image from the inferred noise in the usual way.}
    \label{fig:rxrx1_interpolants}
\end{figure}


Importantly, both $\mathcal{N \rightarrow D}$ and $\mathcal{N \leftrightarrow D}$ are identical at train-time, allowing us to train one model and decide whether or not to use the control images at inference time. Switching to this training strategy gives us a substantial boost in FID to 9.90 (\textit{Config B}; $\mathcal{N \rightarrow D}$) and 13.4 (\textit{Config C}; $\mathcal{N \leftrightarrow D}$), respectively. Our results indicate that, while $\mathcal{C \rightarrow P}$ setups appear to model cell trajectories over time, they are harder to train and may lead to unnecessarily degraded performance. We additionally include results of an undertrained variant of \textit{Config B}, demonstrating that the improvements are not simply due to training for longer. We further discuss the failure of $\mathcal{C \rightarrow P}$ in \cref{sec:addressing_misconceptions}. We additionally include ablations over the form of the interpolant at train-time in \cref{sec:interpolantablation}, finding that the standard linear interpolant attains the best results.

\subsection{Choosing the interpolant} \label{sec:interpolantablation}

We tested two alternatives to the conventional linear interpolant flow matching objective: the generalised variance preserving interpolant \citep{ma2024sit}, and the Brownian bridge interpolant \citep{albergo2023stochastic}. We run all tests on \textit{Config B}, reporting our results in \cref{tab:rxrx1_interpolants}. Both strategies are intuitive, for example, variance preserving interpolants are motivated by shorter paths and Brownian bridges are motivated by using Gaussian (as opposed to Dirac) time-dependent distributions, which may provide a regularising effect. However, we found no evidence that adopting either of these strategies improved performance -- \textbf{we thus continue to use the linear interpolant for all our models.}

\begin{table}[ht]
    \centering
    \caption{\textbf{Interpolant ablations.} All runs use the setup from \textit{Config B} ($\mathcal{N\rightarrow D}$).}
    \begin{tabular}{@{}llccc@{}}
    \toprule
        \textbf{Interpolant} & &\textbf{FID} $\downarrow$ & \textbf{KID} $\downarrow$ \\ \midrule
        
        \textit{Linear} & $tx_1+(1-t)x_0$ & 9.90 & 1.52 \\
        \textit{Variance preserving} & $x_1\mathrm{sin}(\frac{1}{2}\pi t) + x_0\mathrm{cos}(\frac{1}{2}\pi t)$ & 9.99 & 1.42 \\   
        \textit{Brownian bridge, $k=0.1$} & $tx_1+(1-t)x_0 + \epsilon k \sqrt{2t(1-t)}, \quad \epsilon \sim \mathcal{N}(0,1)$ & 14.5 & 6.58 \\
        \textit{Brownian bridge, $k=0.5$} & $tx_1+(1-t)x_0 + \epsilon k \sqrt{2t(1-t)}, \quad \epsilon \sim \mathcal{N}(0,1)$ & 103 & 90.2 \\
        \textit{Brownian bridge, $k=1.0$} & $tx_1+(1-t)x_0 + \epsilon k \sqrt{2t(1-t)}, \quad \epsilon \sim \mathcal{N}(0,1)$ & 180 & 172 \\
     
    \bottomrule
    \end{tabular}
    \label{tab:rxrx1_interpolants}
\end{table}

\subsection{Why do control-to-perturbed flows underperform?} \label{sec:addressing_misconceptions}

At first glance, our results in \cref{tab:rxrx1_perf}, showing that a simple $\mathcal{N \rightarrow D}$ diffusion setup outperforms $\mathcal{C \rightarrow P}$, seem to run contrary to the results reported by \citet{zhang2025cellflux}, who claim that a $\mathcal{C \rightarrow P}$ flow is necessary to model biological batch effects. The first difference comes from our simplified conditioning strategy. Where \citet{zhang2025cellflux} rely on the implicit conditioning given by using the control images to define the source distribution, we simply condition on an embedding of the experiment label, enabling $\mathcal{N \rightarrow D}$ models to capture batch effects\footnote{Note that this does not make our models any less flexible, since we can train with conditioning dropout to jointly learn a conditional and unconditional model \citep{ho2022classifier}.}. Since the experiment label directly identifies the biological batch, conditioning on it provides the same information that a control image implicitly gives. Counterfactual sampling with the $\mathcal{N}\leftrightarrow\mathcal{D}$ flow already allows control images to be used at inference time; if we further wished to generalise to unseen biological batches (which is not a task evaluated in this paper), we could use a foundation model such as OpenPhenom \citep{kenyon-dean2025vitally} to obtain batch embeddings from the control images and condition our model on those. This is a closely related strategy to \citet{atanackovic2025meta}, who develop the \textit{meta flow matching} framework, using distributional embeddings to enable generalising across populations.

Also noteworthy is that RxRx1 contains only 200 control images per experiment -- 160 in the training split and just 40 in the validation split. When the support of the source distribution is finite and small, the ability of a $\mathcal{C} \rightarrow \mathcal{P}$ model to generate diverse images at inference time is limited\footnote{Indeed, assuming deterministic ODE solving, the number of unique outputs a $\mathcal{C} \rightarrow \mathcal{P}$ model can generate per perturbation is at most the number of control images.}. This is not just a limitation of RxRx1 -- similar experimental design is common in all high-throughput screening datasets. Furthermore, we hypothesise that training with an empirical source distribution causes the overfitting issue we observed in RxRx1. In conventional linear interpolant flow matching, the marginal distributions at each $t$ have support over the full space, since $x_t \mid x_1 \sim \mathcal{N}(tx_1, (1-t)^2\mathbf{I})$. However, when both source and target distributions are empirical, there is only supervision on paths connecting each pair in the (finite) cartesian product of the two empirical supports. This set of paths trivially has measure zero in image space. \citet{zhang2025cellflux} work around this by training with noise-augmented control images, finding that without noise augmentation, performance drops significantly. However, worth noting is that the noise augmentation used is so extreme (noise is sampled from $\mathcal{N}(0,\mathbf{I})$, and added to images with range [-1,1]) that it likely obscures much of the information that the control images were supposed to be providing and leads to a distribution shift at inference-time once noise augmentation is turned off. In \citet{Zhangcellfluxv2}, the authors acknowledge these issues with $\mathcal{C \rightarrow P}$ training and combat it by proposing a two-stage training pipeline with additional regularisation. 


\citet{Zhangcellfluxv2} further claim that `\textit{batch effects make a noise-to-distribution formulation inappropriate}'. In this paper, we provide substantial evidence against this claim, demonstrating that noise-to-data setups can be simple, scalable, and performant when implemented correctly.

The comparisons in this section are possible only because \citet{zhang2025cellflux} released extensive reproducibility resources, including their evaluation splits, their generated images, and their MoA classifier. We use these resources extensively throughout this paper (e.g. for verifying our own benchmarks and comparing our models to theirs). Disagreement at this level of technical detail is only possible when prior work is released this completely; we would encourage researchers in the field to follow a similar standard of thoroughness and reproducibility to \citet{zhang2025cellflux}.

\subsection{Coupling}

Conventional flow matching methods train with independent random couplings between samples in the source and target distributions. One alternative, used by \citet{Klein2025.04.11.648220}, is to instead use optimal transport (OT) couplings \citep{pooladian2023multisample,tong2023improving} to reorder samples in each minibatch. OT couplings are motivated by the desire for straighter probability paths between samples, leading to improved performance and lower sampling costs. In \textit{Configs} \{\textit{D,E}\}, we test the OT strategy, finding no benefits in performance or sampling efficiency. See \cref{tab:rxrx1_extended} for full results, including number of function evaluations taken by the ODE sampler; if OT did indeed result in straighter paths, we would expect the sampler to converge in a smaller number of function evaluations (NFE).

\subsection{Architecture} \label{sec:pretraining}

Thus far, we have approximately matched compute and data with \citet{zhang2025cellflux} by selecting the same architecture and training only on RxRx1. For our final models, we will scale compute to build the largest generative-microscopy models to date. We begin by swapping out the ADM U-Net for the larger and more modern diffusion transformer \citep[DiT;][]{peebles2023scalable} to scale our parameter count and FLOPs per sample. Where the ADM architecture had 50M parameters and 200 GFLOPs per forward pass per $96\times96\times6$ microscopy image, our scaled diffusion transformer uses approximately 700M parameters and one TFLOP per forward pass per image.

Unlike latent diffusion on natural images, where diffusion transformers typically need little regularisation, we observed instability when training on microscopy data in $96\times96\times6$ pixel-space. In \cref{sec:arch_ablations}, we find that three small adjustments to the vanilla diffusion transformer help prevent divergence and improve memory footprint and bandwidth. We refer to our final architecture as the \textit{microscopy transformer} (MiT), visualised in \cref{fig:arch_cond_strategy}. Training with our scaled MiT model improves FID to 9.07 on RxRx1.

\subsection{Stabilising pixel space diffusion transformers} \label{sec:arch_ablations}

We tested three types of stability intervention involving changes to optimisation, dropout, and architecture. The simplest and cheapest interventions we found were to apply $10\%$ dropout after the attention projection layer, adding long-range skip connections, and using RMSNorm in place of LayerNorm. We include these changes in our final models. We call this configuration the \textit{Microscopy Transformer} (MiT). 

\begin{table}[ht]
    \centering
    \caption{\textbf{Stabilisation ablations for diffusion transformers.} We consider a setup stable if it does not diverge within the first 100k steps when training from scratch on RxRx1. All models are trained using the $\mathcal{N\rightarrow D}$ flow.} \label{tab:stability_ablations}
    \begin{tabular}[b]{@{}lcl@{}}
    \toprule
        \textbf{Model} & \textbf{Stable} (\cmark/\xmark) & \textbf{Notes} \\ \midrule
        DiT-B/2 & \xmark & Larger variants are stable for longer.\\ \midrule

        \quad \textit{w/ $10^{-5}$ learning rate.} & \xmark & \\
        \quad \textit{w/ $0.01$ AdamW weight decay.} & \xmark & \citet{loshchilovdecoupled}. \\
        \quad \textit{w/ Warmup + inv sqrt decay.} & \xmark & Adapted from \citet{karras2024analyzing}. \\ 
        \quad \textit{w/ Adam $\beta_2 = 0.995$} & \xmark & \citet{berrada2024improved}, Appendix C. \\ \midrule
        \quad \textit{w/ MLP Dropout} & \xmark \\
        \quad \textit{w/ Attention Dropout} & \cmark  & Slightly reduced training throughput. \\
        \quad \textit{w/ Projection Dropout} & \cmark \\ \midrule
        \quad \textit{w/ RMSNorm} & \xmark &   \\
        \quad \textit{w/ RMSNorm + AdaRMS} & \xmark &   \\
        \quad \textit{w/ Block skip} & \xmark \\
        \quad \textit{w/ Block skip + DropPath} & \xmark \\
        \quad \textit{w/ Long range skip} & \cmark & Adapted from \citet{chen2024towards}.\\
    \bottomrule
    \end{tabular}
\end{table}

\paragraph{Optimisation} Somewhat surprisingly, we found that tuning optimisation hyperparameters did not prevent runs from diverging. We tried reducing the learning rate, adding weight decay, using sophisticated schedules, and changing the Adam momentum parameters. However, none of these interventions were successful and \textbf{we chose not to keep any of them in our final setup.} 

\paragraph{Dropout} The conventional DiT is trained without dropout, however we found stability improvements when including it. We tested dropout in three locations: in the MLP, on the attention weights, and after the attention projection layer. We found that only the latter two were useful for stability, and that attention dropout slightly reduced training throughput (we suspect this is due to breaking fused attention kernels). \textbf{We thus use only projection layer dropout in our final setup.}

\paragraph{Architecture} We tested five architecture adjustments: replacing all LayerNorms with RMSNorms; replacing AdaLN with AdaRMS (i.e. removing shift); adding ResNet-style inter-block skip connections; adding DropPath \citep[stochastic depth,][]{huang2016deep}; and adding UNet-style long-range skip connections \citep{chen2024towards}. Of these changes, only the long-range skip connections were helpful for stability. \textbf{We thus include long-range skip connections in our final setup.} Additionally, while using RMSNorm did not help stability, we found that it provided modest memory and throughput improvements and did not observe a performance drop. \textbf{We thus substitute RMSNorm for LayerNorm in our final setup.}

Our final MiT model matches the configuration of the XL/2 diffusion transformer variant (28 depth, 16 heads, 1152 hidden dimensions, 2 patch size). This model has roughly 700M params (excluding conditioning) and approximately one TFLOP per forward pass per $96\times96\times 6$ microscopy image.

\subsection{Transferring the recipe.} 

We made all the design decisions above through ablating on RxRx1. Moving to BBBC021 and JUMP simply requires changing the number of input channels (to three and five, respectively) and changing the size of the perturbation and experiment embedding tables. Re-training our recipe on these benchmarks gives us the state-of-the art results we reported in \cref{fig:main-fig,fig:encoder-fig}.

\section{Virtual screening of unseen molecules} \label{sec:simulation}

In \cref{sec:iid}, we developed the largest and highest-fidelity generative models for microscopy imaging to date. However, for our model to accurately predict effects of \textit{unseen} perturbations, we must develop models that can generalise across conditions. This is a challenging task; expecting a model to generate a truly unique biological state from a new molecule would be analogous to asking a natural image generator to generate a beagle having never seen a dog in its training data. Relying on the assumption that valid phenotypes lie on a shared manifold, we focus on finetuning our model to work with pretrained molecular representations as conditioning. Intuitively, we can think of the pretrained molecule encoder as predicting which mechanism of action the new molecule is likely to act through, then the finetuned generative model as predicting the morphological response from there. If unseen perturbations act through seen mechanisms of action, it may thus be plausible that we can generate realistic morphologies for unseen perturbations.

To evaluate this, we use the publicly available BBBC021 dataset \citep{caie2010high,ljosa2012annotated}, evaluating FID and KID on the seen and unseen perturbation held-out validation splits from \citet{zhang2025cellflux}. The seen split includes held-out images of molecule perturbations seen at training time, whereas the unseen split contains completely new molecule perturbations. We will explore four strategies for molecule conditioning: one-hot, unconditional, Morgan fingerprints, and MolGPS \citep{sypetkowski2024scalability}.

\begin{table*}[ht]
    \centering
    \caption{\textbf{Simulating effects of seen and unseen molecule perturbations on BBBC021.} KID numbers are scaled up by a factor of $10^3$ for ease of interpretation. Note that, since FID and KID are biased metrics, they cannot directly be compared across the seen and unseen columns because they are computed over different numbers of samples.}
    \begin{tabular}{@{}lccccc@{}}
    \toprule
        \multirow{2}{*}{} & \multirow{2}{*}{\textbf{Conditioning}} & \multicolumn{2}{c}{\textbf{Seen}} & \multicolumn{2}{c}{\textbf{Unseen}} \\ 
        \cmidrule(lr){3-4} \cmidrule(lr){5-6}
        & & \textbf{FID} $\downarrow$ & \textbf{KID} $\downarrow$ & \textbf{FID} $\downarrow$ & \textbf{KID} $\downarrow$ \\ \midrule
        MorphoDiff \citep{navidi2025morphodiff} & RDKit2D & 65.8 & 79.9 & - & - \\
        PhenDiff \citep{bourou2024phendiff} & One-Hot & 49.5 & 31.0 & 67.7 & 34.5 \\
  
        IMPA \citep{palma2025predicting} & Morgan Fingerprints & 33.7 & 26.0 & 44.5 & 30.7  \\ 
        
        CellFlux \citep{zhang2025cellflux} & Morgan Fingerprints & 18.7 & 16.2 & 42.0 & 13.1 \\

        CellFluxV2 \citep{Zhangcellfluxv2} & Morgan Fingerprints & 7.90 & 2.30 & 19.2 & \textbf{8.00} \\
        
        \midrule
        
        Ours -- \textit{one-hot} & One-Hot & \textbf{4.03} & 0.31 & -  & - \\
        Ours -- \textit{unconditional} & Unconditional & 4.63 & 0.68 & 22.7 & 13.5 \\
        Ours -- \textit{Morgan fingerprints} & Morgan Fingerprints & 4.16 & 0.18 & 19.9 & 12.0 \\
        Ours -- \textit{MolGPS} & MolGPS & 4.12 & \textbf{0.16}  & \textbf{18.5} & 9.94 \\        
    \bottomrule
    \end{tabular}
    \label{tab:bbbc_perf}
\end{table*}

\subsection{Seen molecules}

We begin by taking the recipe we developed in \cref{sec:iid} and re-training it from scratch on the BBBC021 dataset. On the seen molecules split, we attain an FID of 4.03 and KID of 0.31, which is broadly consistent with the 2-10$\times$ improvements over prior state-of-the-art that we observed on RxRx1 and JUMP. We report results in \cref{tab:bbbc_perf}, and provide side-by-side comparisons of real and generated images in \cref{fig:bbbc_iid_samples_extra}.

\subsection{Unseen molecules: exploring embeddings}

Since our model uses one-hot perturbation embeddings by default (see \cref{fig:arch_cond_strategy}), it cannot be directly used to predict effects of unseen perturbations. We thus explore three strategies for virtual screening of unseen molecules. We present further comparisons of real and generated images in \cref{fig:bbbc_ood_samples_extra}.

\paragraph{Unconditional} The first method we consider for simulating unseen perturbations is to simply sample unconditionally from our model (this is possible because we trained our model with conditioning dropout to enable classifier-free guidance \citep{ho2022classifier}). While this is a naive strategy that does not allow for control over the sampled phenotypes, it attains an FID of 22.7, which surprisingly beats all prior methods barring CellFluxV2. One explanation for this result is that the reduction in sampling control is more than compensated for by the capabilities of our base model, which simply generates more realistic images.

\paragraph{Morgan fingerprints} Next, we investigate Morgan fingerprint embeddings, as used by IMPA, CellFlux, and CellFluxV2. Following our conditioning strategy in \cref{fig:arch_cond_strategy}, we freeze our base model and train a 100M parameter adaptor to map between Morgan fingerprints and the base model's conditioning space. This setup attains an FID of 19.9, approximately on-par with CellFluxV2, the prior state-of-the-art method. Morgan fingerprint conditioning only results in marginally improved FID over unconditional sampling, perhaps indicating that these embeddings provide the model little additional information relevant to predicting the morphological effects of drugs.

\paragraph{MolGPS} Finally, we repeat the previous experiment, replacing Morgan fingerprints with MolGPS~\citep{sypetkowski2024scalability} embeddings of the molecules. MolGPS is a pretrained 3B-parameter graph-transformer designed for molecular representation learning. We find that conditioning on MolGPS embeddings leads to our best results, with a state-of-the-art FID of 18.5 and a KID of 9.94.

Potentially noteworthy is how our model attains 2-10$\times$ improvements over state-of-the-art on seen perturbations, however, even with MolGPS (one of the most advanced molecular representation learning models) conditioning, the gap between our model's performance and other methods is much smaller. This may indicate that molecular representations are the main bottleneck in improving virtual screening performance.

\subsection{Qualitative assessment of biological fidelity}

Beyond generative modelling metrics, it is important that generated samples reproduce the specific morphological phenotypes associated with chemical perturbations. We analyse the biological plausibility of generated images for the unseen AZ841 perturbation, an Aurora kinase inhibitor.

The generated composites in \cref{fig:az841_review} closely match the real images in nuclear and cytoplasmic organisation, preserving F-actin architecture (phalloidin; red),
microtubule organisation ($\beta$-tubulin; green), and nuclear size and shape (DNA; blue). This is consistent with a non-lethal, cell-cycle–modulating response and is in line with established effects of Aurora kinase inhibition, which induces mitotic arrest and cytokinesis failure leading to polyploidy, multinucleation, micronuclei formation, and increased chromatin heterogeneity \citep{gautschi2008aurora,goldenson2015aurora,bavetsias2015aurora}.

\begin{figure*}[htbp]
    \centering
    \includegraphics[width=\linewidth]{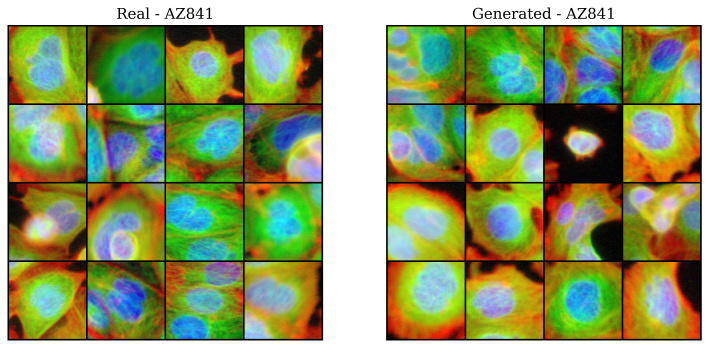}

    \caption{\textbf{Randomly sampled real and generated images for virtual screening on BBBC021.} The AZ841 molecule was not seen by the model during training. Samples are generated using our model (with Morgan fingerprints).}
    \label{fig:az841_review}
\end{figure*}

Both real and generated samples display a heterogeneous nuclear phenotype with variably-sized nuclei, including condensed mitotic forms and enlarged or irregular nuclei, as well as altered nuclear packing. The generated images reproduce the major spatial co-localisation patterns across channels and the degree of single-cell morphological variability observed in the real data, indicating that the principal perturbation-specific cellular morphology associated with Aurora kinase blockade is captured. Residual discrepancies are confined mainly to fine-scale structural detail, with smoother chromatin and cytoskeletal textures, as well as the occasional minor intensity-artefact in generated samples.

\section{Datasets, evaluation, and training details} \label{sec:datadetails}

\subsection{Metrics}

\paragraph{Fréchet inception distance} In line with the wider image-generative modelling community, we use Fréchet Inception Distance (FID) as our primary measure of distance between distributions. This involves embedding the real and generated sets of images with the 2048-dimensional Inception-v3 pool features \citep{szegedy2016rethinking}, then comparing the distributions of embeddings under a Gaussian assumption \citep{heusel2017gans}. Since the Inception encoder requires $299\times299\times3$ images as input, we preprocess all images by projecting to RGB (see details below) and resizing with bilinear interpolation. Importantly, FID is a \textit{biased} metric \citep{chong2020effectively}, meaning that the metric has no absolute scale and can only be used for relative comparisons of figures computed from the same number of samples against the same reference set of real images.

\paragraph{Kernel inception distance} A complementary alternative metric to FID is Kernel Inception Distance (KID), a squared MMD distance between the embedded real and generated datasets computed with a cubic kernel. Compared to FID, KID is unbiased\footnote{While KID is technically unbiased, in practice it does show sensitivity to number of samples on our datasets. See the ablation in Appendix H of \citet{zhang2025cellflux} for more detail.} and makes no Gaussian assumption. In our tests, we observed that KID typically favours higher fidelity samples at the expense of diversity. Following the published benchmarks of \citep{zhang2025cellflux}, we compute KID with a subset size of 100. We use the same RGB projection and resizing preprocessing for KID as we do for FID.

\paragraph{Real image scores} Following recent work in natural imaging \citep{yang2026representation}, for each dataset, we compute FID between the evaluation set and a composition-mirrored subset of the training set (\cref{tab:real_image_scores})\footnote{Note that we cannot include real image scores for the unseen compound task on BBBC021 since, by construction, we cannot draw a perturbation-matched sample from the training set.}. While this is not a hard upper bound on model performance, it provides a baseline number that gives (otherwise unitless) FID scores context. The result that our model scores very close to the real images across all datasets provides strong evidence that the metric is becoming saturated for the purpose of measuring performance of generative microscopy models. We compute an approximate upper bound on model performance for KID in the same manner as our FID upper bound; since KID is not a biased metric, these scores are all close to zero.

\begin{table}[ht]
      \centering
      \caption{\textbf{Real image scores for all distance metrics.} We compute the distance between the evaluation set and images drawn from the training set (disjoint from the evaluation set). We match number of images, composition of perturbations, and composition of experiments between both datasets.}
      \label{tab:real_image_scores}
      \begin{tabular}{@{}lccc@{}}
      \toprule
          & \textbf{BBBC021} & \textbf{RxRx1} & \textbf{JUMP} \\ \midrule
          FID (Inception-v3)                   & 3.73 & 8.11    & 2.30 \\
          KID (Inception-v3, $\times 10^{3}$)  & 0.01 & $-0.08$ & 0.12 \\
          FDD (DINOv2-L/14)                    & 17.5 & 32.7    & 7.90 \\
          KDD (DINOv2-L/14, $\times 10$)       & 0.02 & $-0.10$ & $-0.03$ \\
          FD (OpenPhenom, $\times 10^{3}$)     & 22.6 & 11.5    & 4.67 \\
      \bottomrule
      \end{tabular}
  \end{table}
  
\paragraph{Alternative encoders} The classic Inception encoder has known issues \citep{stein2023exposing}, so we additionally compute alternative variants using DINOv2-L/14 \citep{oquab2024dinov} and OpenPhenom \citep{kenyon-dean2025vitally} for our results in \cref{fig:encoder-fig}a. OpenPhenom is a microscopy-trained, channel-agnostic masked autoencoder, so we run it on the native channels instead of projecting to RGB space. Both require bilinear resizing to their native size of $224\times224$px.

\paragraph{Mechanism of action classification} The BBBC021 dataset comes with annotations of the Mechanism of Action (MoA) of each small molecule perturbation. We report MoA classification scores using the pretrained classifier released by \citet{zhang2025cellflux} (trained on real images) to predict MoA on our generated images. We report accuracy, macro-F1, and weighted-F1 (excluding the classes that do not appear in the evaluation split). The real image scores reported for MoA classification represent the performance of the classifier on the real evaluation split. Note that this is only an approximate upper bound on model performance; for example, a model generating low-diversity prototypical images of each perturbation may attain higher MoA classification scores than the real images.

\paragraph{CellProfiler metrics} We use CellProfiler \citep{carpenter2006cellprofiler}, with the \texttt{cp\_measure} package \citep{munoz2025cpmeasure} to obtain six interpretable features for the real and generated images:

\begin{itemize}
    \item \texttt{Area}
    \item \texttt{Solidity}
    \item \texttt{Intensity\_MeanIntensity\_\_DNA}
    \item \texttt{RadialDistribution\_FracAtD\_1of4\_\_DNA}
    \item \texttt{Contrast\_3\_00\_256\_\_DNA}
    \item \texttt{Entropy\_3\_00\_256\_\_DNA}
\end{itemize}

The RxRx1, BBBC021, and JUMP benchmarks we use are restricted to $96\times96$px crops, centred on cell nuclei. Since these crops often clip cytoplasm and cell boundaries, we restrict our CellProfiler analysis to nucleus features only. We compute nuclear masks by Otsu thresholding the DNA channel and taking the connected component covering the crop centre.

In \cref{fig:encoder-fig}b, we plot the distribution (over crops) of each feature for the generated and real (evaluation set) images. To give the subplots a common x-axis scale, for each dataset-feature combination, we transform both real and generated distributions by subtracting the mean of the real distribution and dividing by the standard deviation of the real distribution (such that the real distribution in each subplot is normalised). We additionally annotate \cref{fig:encoder-fig}b with the standardised mean difference between generated and real images, computed as $(\bar{x}_{\text{gen}} - \bar{x}_{\text{real}}) / \sigma_{\text{real}}$.
  
\subsection{The RxRx1 dataset} 

The RxRx1 dataset consists of $125$k $2048\times2028\times6$ images of wells collected with the Cell Painting protocol \citep{bray2016cell}. Each well contains approximately $750$ cells treated with one of $1108$ distinct siRNA gene knockdown perturbations. The dataset was collected over $51$ identically designed biological experiments (sometimes referred to as biological `batches' or `contexts'), with \citet{sypetkowski2023rxrx1} noticing significant variation, or `batch effects' across experiments. We use identical preprocessing to \citet{palma2025predicting}\footnote{The preprocessed dataset is available for download at: \url{zenodo.org/records/8307629}.}, resulting in $170,943$ images of the U2OS cell line, with $1071$ unique perturbations from three experiments. Each image is illumination-corrected and cropped to $96 \times 96 \times 6$ resolution centred on a cell nucleus.

For consistency with prior benchmarks, we evaluate FID and KID on the held-out validation split defined by \citet{zhang2025cellflux}, consisting of 2942 images from a subset of 100 perturbations, and 120 additional untreated images for use as controls. For evaluation, we generate 2942 images, exactly mirroring the composition of perturbations and experiments as the validation split. All models are trained on the native 6-channel images, but evaluated in RGB space using the projection function in Listing \ref{listing:cp2rgb}. We verified our evaluation pipeline with the generated images from \citet{zhang2025cellflux}\footnote{Available for download at: \url{huggingface.co/suyc21/CellFlux}}, reproducing their reported FID to within $0.01$ ($33.01$) and KID to within $0.0003$ ($0.0235$).

\begin{lstlisting}[language=Python, caption=Cell-paint to RGB projection used for RxRx1 model evaluation., label={listing:cp2rgb}]
import torch

def cp2rgb(images):
    """Projects 6-channel cell-paint to 3-channel RGB.
    Input: torch.float32 tensor, shape B,6,H,W.
    Output: torch.float32 tensor, shape B,3,H,W.
    """
    weights = torch.tensor(
        [
            [0.0, 0.0, 1.0],  # Hoechst -> Blue
            [0.0, 1.0, 0.0],  # ConA -> Green
            [1.0, 0.0, 0.0],  # Phalloidin -> Red
            [0.0, 0.5, 0.5],  # Syto14 -> Cyan
            [0.5, 0.0, 0.5],  # MitoTracker -> Magenta
            [0.5, 0.5, 0.0],  # WGA -> Yellow
        ],
        dtype=images.dtype,
        device=images.device,
    )
    images_rgb = torch.einsum("bchw,cn->bnhw", images, weights)
    return images_rgb.contiguous()
\end{lstlisting} 

\subsection{The BBBC021 dataset}

The BBBC021 dataset \citep{caie2010high,ljosa2012annotated} consists of 40k $1280\times 1024\times3$ three-channel fluorescence images of wells collected from a phenotypic screen of 113 small molecules on the MCF7 cell type. We use identical preprocessing to \citet{palma2025predicting}, resulting in 96k $96\times96\times3$ illumination-corrected crops, centred on cell nuclei. For evaluation, we treat the three fluorescence channels as RGB images directly (no projection required). We use the splits from \citet{zhang2025cellflux}, consisting of 5237 held-out images of the following seen molecule perturbations: 

\begin{multicols}{3}
\begin{itemize}
    \item \texttt{Alsterpaullone}
    \item \texttt{AZ138}
    \item \texttt{AZ258}
    \item \texttt{Bryostatin}
    \item \texttt{Camptothecin}
    \item \texttt{Chlorambucil}
    \item \texttt{Cisplatin}
    \item \texttt{Colchicine}
    \item \texttt{Cytochalasin B}
    \item \texttt{Etoposide}
    \item \texttt{Floxuridine}
    \item \texttt{Methotrexate}
    \item \texttt{Mevinolin}
    \item \texttt{Mitomycin C}
    \item \texttt{Mitoxantrone}
    \item \texttt{PD-169316}
    \item \texttt{PP-2}
\end{itemize}
\end{multicols}

Similarly, the out-of-distribution validation set has 1979 images of the following unseen perturbations (we remove all images of these perturbations from the training set):
\begin{multicols}{2}
\begin{itemize}
    \item \texttt{AZ841}
    \item \texttt{Cyclohexamide}
    \item \texttt{Cytochalasin D}
    \item \texttt{Docetaxel}
    \item \texttt{Epothilone B}
    \item \texttt{Lactacystin}
    \item \texttt{Latrunculin B}
    \item \texttt{Simvastatin}
\end{itemize}
\end{multicols}

We verified our evaluation pipeline with the generated images from \citet{zhang2025cellflux}, reproducing their reported seen FID to within 0.2 (18.5) and KID to within 0.001 (0.0152).

\subsection{The JUMP dataset}

The JUMP \texttt{cpg0000} dataset \citep{chandrasekaran2024three} consists of Cell Painting \citep{bray2016cell} images of a heterogeneous mix of small molecule, CRISPR, and ORF overexpression perturbations. We use identical preprocessing to \citet{palma2025predicting}\footnote{We use the \texttt{cpg0000\_u2os\_normalized\_segmented\_large} variant. A plate-corrected (`transported') variant of the same crops is also released; we use the uncorrected one, matching \citet{zhang2025cellflux}.}, resulting in $424$k illumination-corrected $96\times96\times5$ crops of the U2OS cell line, each centred on a cell nucleus. The crops come from ten plates, which we use as the biological batch label.

For consistency with prior benchmarks, we evaluate FID and KID on the held-out validation split defined by \citet{zhang2025cellflux}, consisting of $9077$ images from a subset of $100$ randomly selected perturbations ($40$ compound, $44$ CRISPR and $16$ ORF). For evaluation, we generate $9077$ images, exactly mirroring the composition of perturbations and plates as the validation split.

All models are trained on the native 5-channel images, but evaluated in RGB space, by simply taking the first three channels. We verified our evaluation pipeline with the generated images from \citet{zhang2025cellflux}, reproducing their reported FID to within $0.03$ ($8.97$) and KID to within $0.0001$ ($0.0062$). We additionally confirmed that the $100$ perturbations in the split we constructed are exactly the $100$ perturbations for which they released generations.

\subsection{Training details and extended results} \label{sec:trainingdetails}

We train all models on the task of continuous-time velocity prediction. We train in BF16-mixed precision on 8 NVIDIA H100 GPUs (data parallel) for 100k steps using the Adam optimiser \citep{kingma2014adam}. We run inference on all models using the \citet{karras2024analyzing} power-function EMA formulation and standard classifier-free guidance \citep{ho2022classifier}. We sample using the \citet{dormand1980family} 5th-order Runge-Kutta ODE solver (DoPri5), which typically uses $\approx 300$ function evaluations ($\times2$ when using guidance $\not = 1.0$). These are generally sensible inference-time defaults, however, it is likely possible to squeeze more performance out of our base models by jointly tuning the EMA and guidance \citep{karras2024analyzing}, as well as using more solver steps (or an SDE solver in place of our ODE solver \citep{ma2024sit}). Unless otherwise stated, all models use the training and inference parameters in \cref{tab:rxrx1_hyperparams}.

\begin{table}[ht] 
\caption{\textbf{Training and inference hyperparameters for all RxRx1 models.}}\label{tab:rxrx1_hyperparams}
\centering
\begin{tabular}{@{}lll@{}}
\toprule
 
\textbf{Optimiser} & Adam \{\texttt{lr}: $10^{-4}$, $\beta_1$: $0.9$, $\beta_2$: $0.999$, $\epsilon$: $10^{-8}$\} \\
 
\textbf{LR schedule} & Constant \\

\textbf{Gradient norm clipping} & 0.5 \\
 
\textbf{Global batch size} & 256 \\
  
\textbf{Augmentation} & Uniform dequantisation $\sim \mathcal{U}(0,1/256)$\\ & Random flipping (horizontal and vertical)   \\ \midrule

\textbf{Solver} & DoPri5 \{$\texttt{rtol} = \texttt{atol} = 10^{-5}$\} \\

\textbf{EMA} & Power function EMA \{$\sigma_{\text{rel}} = 0.01$\} \\

\bottomrule
\end{tabular}
\end{table}

\subsubsection{RxRx1 training configurations}

We provide additional notes on each of the \cref{tab:rxrx1_perf} configurations below, as well as extended results of a sweep over guidance strength in \cref{tab:rxrx1_extended}. 

 \paragraph{Config A} We match the setup of \citet{zhang2025cellflux} by randomly coupling control and perturbation images from the same biological experiment at train time and augmenting control images with noise sampled from $ \mathcal{N}(0,\mathbf{I})$ with probability $0.5$. We found that this setup ($\mathcal{C \rightarrow P}$) was susceptible to overfitting and that early stopping after 30k training steps produced the best results. Our final setup attains an FID of 29.5, which is a marginal improvement upon \citet{zhang2025cellflux}. We attribute this improvement to our simpler conditioning, combined with tuning the guidance strength. We had best results sampling using classifier-free guidance with strength 2.0.

\paragraph{Config B} We perform standard flow matching, defining the source distribution as $\mathcal{N}(0,\mathbf{I})$. The target distribution is the data distribution, with control images treated as an extra perturbation class. We train for the full 100k steps (approximately 80 H100 hours) and had best results sampling with guidance strength 1.0 (no guidance). Despite being the simplest of our configurations, it attained a strong FID of 9.90 and KID of 1.52. We additionally report the results of an undertrained variant, stopped after 25k steps. This is the cheapest configuration we trained, yet still beats prior state-of-the-art, with an FID of 12.3 and KID of 3.90. 

\paragraph{Config C} We take the checkpoint from \textit{Config B} and perform \textit{counterfactual}-style generation. This involves starting inference with a control image, solving the reverse ODE to get the corresponding noise, then solving the forward ODE with the new perturbation conditioning. Our best results came from sampling with guidance strength 1.0 (no guidance). This configuration gives a slightly worse FID of 13.4 compared to \textit{Config B}, but slightly improved KID of 1.14. Similar to \textit{Config A}, we found that the small number of available control images led to limited diversity in the generated images (although there was no overfitting issue this time, since we trained on the $\mathcal{N \rightarrow D}$ flow).

\paragraph{Config D} We repeat the experiment from \textit{Config B}, but train with minibatch optimal transport couplings \citep{pooladian2023multisample,tong2023improving}. This involves re-ordering samples in each minibatch to minimise the average earth-mover distance between the source-target pairings. We perform the reordering independently on each GPU process with a local batch size of 32. While minibatch optimal transport is intuitively motivated by the desire for straighter paths and better performance, we found no performance improvement over \textit{Config B} and only a marginal change in number of function evaluations (NFE) when using the DoPri5 adaptive step size solver, indicating that the curvature of the paths is similar. Best results came from sampling with guidance strength 1.5.

\paragraph{Config E} We repeat the counterfactual sampling experiment from \textit{Config C}, but this time using the optimal transport checkpoint from \textit{Config D}. While \citet{ribeiro2025counterfactual} find that OT is beneficial for counterfactual models at small scales, we found no benefit to using optimal transport couplings (FID of 13.6 and KID of 2.15). This is likely due to the inherent challenges in scaling minibatch OT to large datasets. Best results involved sampling with guidance strength 1.5.

\paragraph{Final model} We take the best of our previous ablations (\textit{Config B}) and use our scaled microscopy transformer (MiT-XL/2) architecture in place of the ADM U-Net. We train for 150k steps on 16 GPUs with a local batch size of 16 per GPU (approximately 500 H100 hours). This setup improved FID to 9.07 and KID to 0.84 (guidance strength 1.0).

\begin{table}[ht]
    \centering
    \caption{\textbf{Extended results of RxRx1 experiments.} For each configuration, we include a coarse-grained sweep over guidance strength $\{1.0, 1.5, 2.0\}$. In addition to conventional FID and KID, we report DINOv2 alternatives, as proposed by \citet{stein2023exposing} (FDD and KDD respectively). We include the mean number of function evaluations (NFE) used per batch for sampling with each configuration. Notice how using guidance strength $\not= 1$ approximately doubles NFE, since it requires two forward passes per step. Similarly, using counterfactual generation doubles it again since it requires solving two ODEs. Rows marked with * represent the configurations chosen for \cref{tab:rxrx1_perf}. KID and KDD numbers are scaled up by factors of $10^3$ and 10, respectively.}
    \begin{tabular}{@{}lcccccc@{}}
    \toprule
        & \textbf{NFE} & \textbf{FDD} $\downarrow$ & \textbf{KDD} $\downarrow$ &\textbf{FID} $\downarrow$ & \textbf{KID} $\downarrow$ \\ \midrule
        
        \textit{Config A ($\mathcal{C \rightarrow P}$)} \\
        
        \quad \textit{guidance strength 1.0} & 381 & 272 & 14.9 & 34.0 & 16.7 \\
        \quad \textit{guidance strength 1.5} & 826 & 266 & 14.1 & 31.2 & 14.5 \\
        \quad \textit{guidance strength 2.0*} & 918 & 261 & 13.8 & 29.5 & 14.0 \\ \midrule
        
        \textit{Config B ($\mathcal{N \rightarrow D}$)} \\
        \quad \textit{guidance strength 1.0*} & 299 & 51.0 & 1.75 & 9.90 & 1.52 \\
        \quad \textit{guidance strength 1.5} & 634 & 49.3 & 1.43 & 10.2 & 1.61 \\
        \quad \textit{guidance strength 2.0} & 661 & 57.9 & 1.92 & 11.7 & 2.87 \\ \midrule
        
        \textit{Config C ($\mathcal{N \leftrightarrow D})$}\\
        \quad \textit{guidance strength 1.0*} & 646 & 63.6 & 0.85 & 13.4 & 1.14 \\
        \quad \textit{guidance strength 1.5} & 1366 & 65.1 & 1.22 & 13.9 & 1.98 \\
        \quad \textit{guidance strength 2.0} & 1440 & 81.2 & 2.09 & 16.6 & 4.64 \\ \midrule
        
        \textit{Config D ($\mathcal{N \rightarrow D}$ + OT)} \\
        \quad \textit{guidance strength 1.0} & 293 & 58.3 & 2.45 & 10.9 & 2.13 \\
        \quad \textit{guidance strength 1.5*} & 622 & 53.7 & 1.85 & 10.5 & 1.61 \\
        \quad \textit{guidance strength 2.0} & 653 & 52.9 & 1.55 & 10.6 & 1.41 \\ \midrule

         \textit{Config E ($\mathcal{N \leftrightarrow D}$ + OT)} \\
         \quad \textit{guidance strength 1.0} & 643 & 80.9 & 1.76 & 17.3 & 4.59 \\
         \quad \textit{guidance strength 1.5}* & 1392 & 63.4 & 1.01 & 13.6 & 2.15 \\
         \quad \textit{guidance strength 2.0} & 1466 & 63.1 & 1.05 & 13.7 & 2.23 \\ \midrule
        
        \textit{Final model (scaled MiT-XL/2)} \\
        \quad \textit{guidance strength 1.0*} & 313 & 41.8 & 0.89 & 9.07 & 0.84  \\
        \quad \textit{guidance strength 1.5} & 693 & 41.0 & 0.81 & 9.11 & 0.75 \\
        \quad \textit{guidance strength 2.0} & 754 & 45.4 & 0.87 & 9.63 & 1.00 \\
    \bottomrule
    \end{tabular}
    \label{tab:rxrx1_extended}
\end{table}

\subsubsection{BBBC021 training configurations}

We provide additional notes on each of the \cref{tab:bbbc_perf} configurations below, as well as extended results in \cref{tab:bbbc_extended}.

\paragraph{Our model (one-hot)} Similar to our model setup from RxRx1, we train MiT-XL/2 from scratch for 100k steps on 16 GPUs with a local batch size of 16 per GPU. Since we are using one-hot perturbation embeddings, we are only able to generate images for seen perturbations. We use guidance strength 1.0, attaining a state-of-the-art FID of 4.03 and KID of 0.31.

\paragraph{Our model (unconditional)} We take the checkpoint from \textit{our model (one-hot)} and perform unconditional sampling by replacing the perturbation labels with the null token (the model was trained with conditioning dropout to enable this method of unconditional sampling). As expected, our unconditional model achieves a slightly worse FID and KID on seen perturbations to the conditional variant (4.63 and 0.68, respectively). On unseen perturbations, it attains an FID of 22.7 and KID of 13.5, beating all prior results except CellFluxV2.  

\paragraph{Our model (with Morgan fingerprints)} We take the checkpoint from \textit{our model (one-hot)} and replace the one-hot encoding with Morgan fingerprint embeddings of the molecules. We freeze the main model (except for the AdaLN modulation layers) and train a 100M parameter transformer adaptor\footnote{We implement the adaptor using the same MiT transformer architecture as the main model, with 6 depth, 6 heads, and 1152 hidden dimensions. We use a continuous embedding of the compound dosage as the conditioning tokens for the adaptor itself. We treat Morgan fingerprints as single tokens, effectively making the adaptor equivalent to a large MLP.} to map the Morgan fingerprints and compound dosage to the conditioning space of the model. We train for 10k steps on 8 GPUs with a batch size of 16 per GPU. 

\paragraph{Our model (with MolGPS)} We repeat the previous experiment, using MolGPS \citep{sypetkowski2024scalability} embeddings of the compounds in place of Morgan fingerprints. This setup achieved the best FID on unseen compounds to date (18.5).

\begin{table}[ht]
    \centering
    \caption{\textbf{Extended results of BBBC021 experiments.} We use guidance strength 1.0 for all conditional models (the unconditional variant uses 0.0). KID and KDD numbers are scaled up by factors of $10^3$ and 10, respectively.}
    \begin{tabular}{@{}lcccccc@{}}
    \toprule
        & \textbf{NFE} & \textbf{FDD} $\downarrow$ & \textbf{KDD} $\downarrow$ &\textbf{FID} $\downarrow$ & \textbf{KID} $\downarrow$ \\ \midrule

        \textbf{Seen molecules} \\
        \quad\textit{Ours (one-hot)} & 239 & \textbf{23.2} & 0.63 & \textbf{4.03} & 0.31 \\

        \quad \textit{Ours (uncond.)} & 242 & 25.0 & 0.58 & 4.63 & 0.68\\
        \quad \textit{Ours (Morgan)} & 241 & 23.9 & 0.50 & 4.16 & 0.18 \\
        \quad \textit{Ours (MolGPS)} & 240 & \textbf{23.2} & \textbf{0.46} & 4.12 & \textbf{0.16} \\ \midrule

        \textbf{Unseen molecules} \\
        \quad\textit{Ours (one-hot)} & - & - & - & - & - \\

        \quad \textit{Ours (uncond.)} & 241 & 117 & 5.76 & 22.7 & 13.5 \\
        \quad \textit{Ours (Morgan)} & 245 & 113 & 6.07 & 19.9 & 12.0 \\
        \quad \textit{Ours (MolGPS)} & 245 & \textbf{106} & \textbf{5.42} & \textbf{18.5} & \textbf{9.94} \\
        \bottomrule
    \end{tabular}
    \label{tab:bbbc_extended}
\end{table}

\subsubsection{JUMP training configurations}

Similar to our model setup from RxRx1 and BBBC021 (seen molecules), we train MiT-XL/2 from scratch for 100k steps on 16 GPUs with a local batch size of 16 per GPU. We use the plate annotation in JUMP as the batch/experiment label to condition on.

\subsubsection{Estimating training FLOPs} \label{sec:flops_estimation}

We assume that each backward pass is approximately $2\times$ the compute of a forward pass and thus estimate total training compute in ExaFLOPs with the formula: FLOPs per forward pass per image $\times$ global batch size $\times$ steps $\times$ 3 $\times 10^{-18}$. Following convention in the machine learning literature, we assume fused multiply-accumulate operations (i.e. reported FLOPs $\approx \frac{1}{2}$ True FLOPs). We include notes below for each of our estimated values in \cref{tab:rxrx1_perf}.

\paragraph{CellFlux} \citet{zhang2025cellflux} use the ADM U-Net \citep{dhariwal2021diffusion}. Cloning their exact configuration, we found that it uses approximately 200 GFLOPs per forward pass per $96\times96\times 6$ image. \citet{zhang2025cellflux} report training for 100 epochs with a global batch size of 128. This leads to our 7.83 ExaFLOP estimate for total training cost.

\paragraph{CellFluxV2} \citet{Zhangcellfluxv2} use a latent diffusion transformer (DiT-XL/2) at $28\times28 \times 8$ resolution. Cloning and running the model, we found that it uses approximately 88 GFLOPs per forward pass per image, excluding the VAE (which we may neglect if we assume training with cached latents). \citet{Zhangcellfluxv2} report training for 200 epochs with a global batch size of 128, leading to our 6.89 ExaFLOP estimate.

\paragraph{IMPA} We cloned the IMPA model\footnote{Available at \url{github.com/theislab/IMPA}.} from \citet{palma2025predicting}, finding that the generator network uses approximately 7.8 GFLOPs per forward pass per $96\times96\times 6$ image, with the discriminator using approximately 2.3 GFLOPs. \citet{palma2025predicting} do not provide any details on training time, steps, or batch size, nor do \citet{zhang2025cellflux} report any details about their reproduction. If we assume roughly equivalent batch size and epochs to CellFlux (i.e. only the model was changed between experiments), we may estimate that the total training compute is approximately 5\% of CellFlux, or 0.39 ExaFLOPs.

\paragraph{PhenDiff} \citet{bourou2024phendiff} provide few details about their model or training, however, their repository\footnote{Available at \url{github.com/thethomasboyer/PhenDiff}.} appears to use the Stable Diffusion 2.1 model and VAE combination, with inputs resized to $128\times128\times 6$. This uses roughly 50 GFLOPs per forward pass per sample. From their config files, they appear to train for 50k steps, with a global batch size of 128 (2 GPUs $\times$ 64 per GPU). These numbers lead to our estimate of 0.96 ExaFLOPs training compute.

\clearpage
\newpage

\section{Additional qualitative samples} \label{sec:qualitativeresults}

\begin{figure}[htbp]
    \centering
    \includegraphics[width=0.9\linewidth]{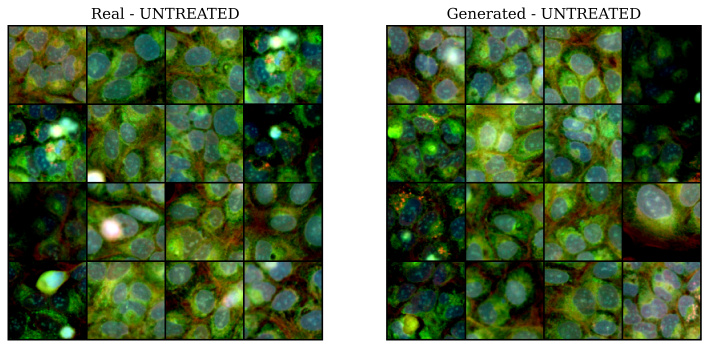}
    \includegraphics[width=0.9\linewidth]{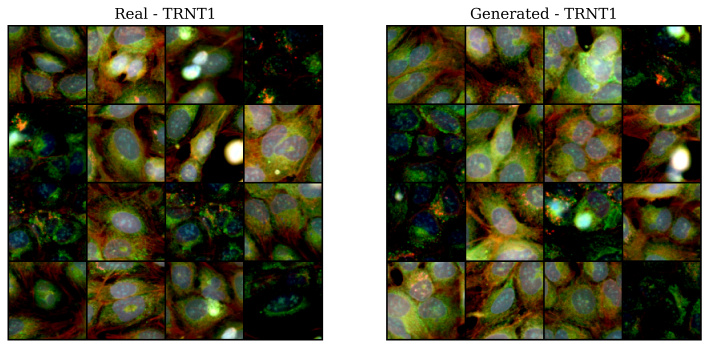}
    \includegraphics[width=0.9\linewidth]{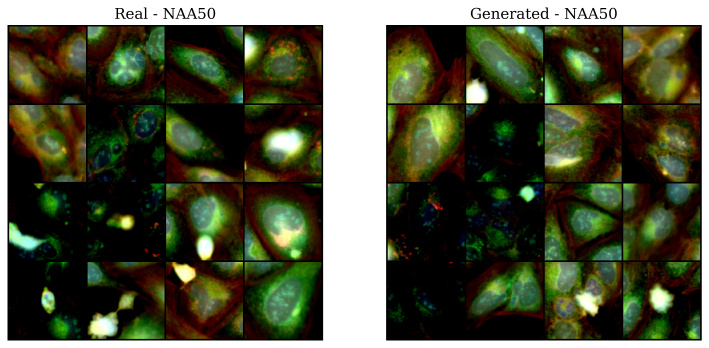}

    \caption{\textbf{Randomly sampled real and generated images for three siRNA gene perturbations in RxRx1.} Samples are generated using our final configuration (no counterfactual sampling), matching the composition of perturbations and biological experiments to the randomly sampled real images.}
    \label{fig:rxrx1_samples_extra}
\end{figure}

\clearpage
\begin{figure}[htbp]
    \centering
    \includegraphics[width=0.9\linewidth]{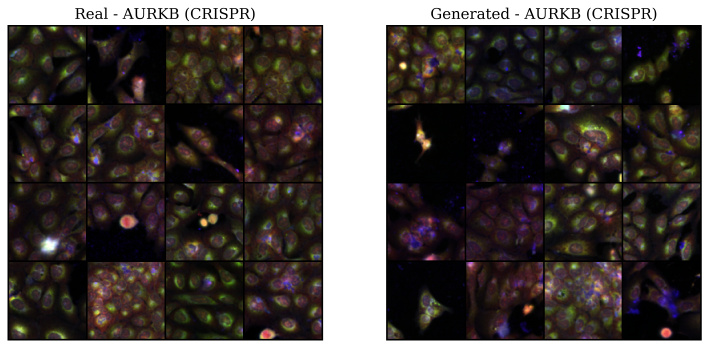}
    \includegraphics[width=0.9\linewidth]{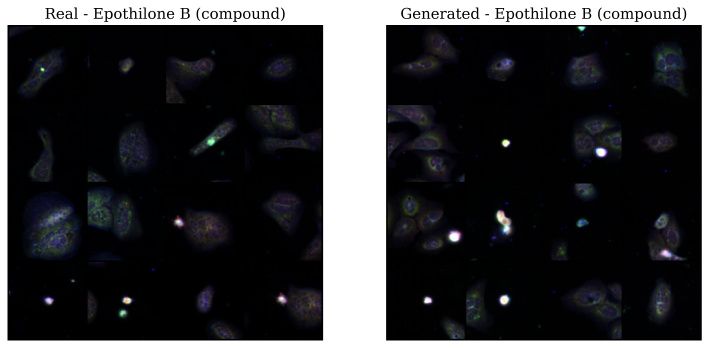}
    \includegraphics[width=0.9\linewidth]{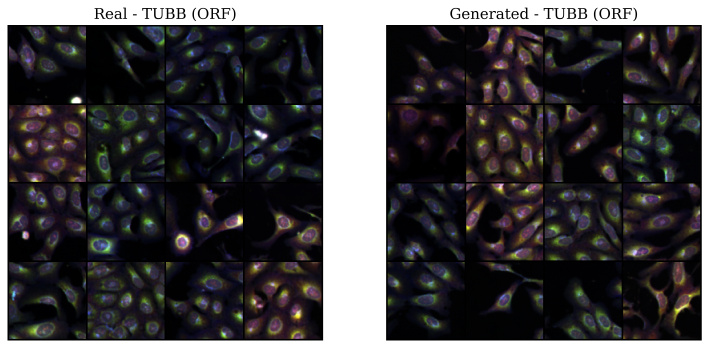}

    \caption{\textbf{Randomly sampled real and generated images for three perturbations in JUMP.} Samples are generated using our final configuration (re-trained from scratch on JUMP), matching the composition of perturbations and biological experiments to the randomly sampled real images.}
    \label{fig:jump_samples_extra}
\end{figure}

\clearpage
\begin{figure}[htbp]
    \centering
    \includegraphics[width=0.9\linewidth]{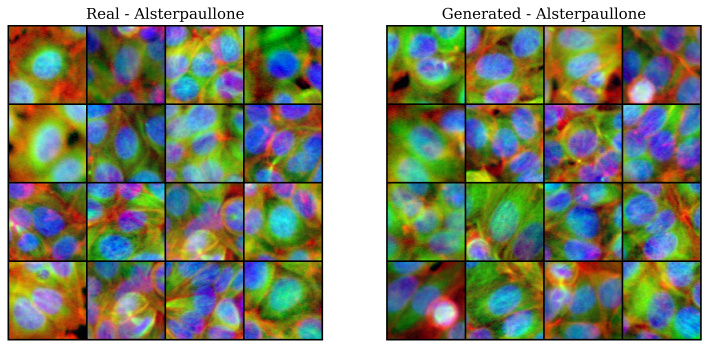}
    \includegraphics[width=0.9\linewidth]{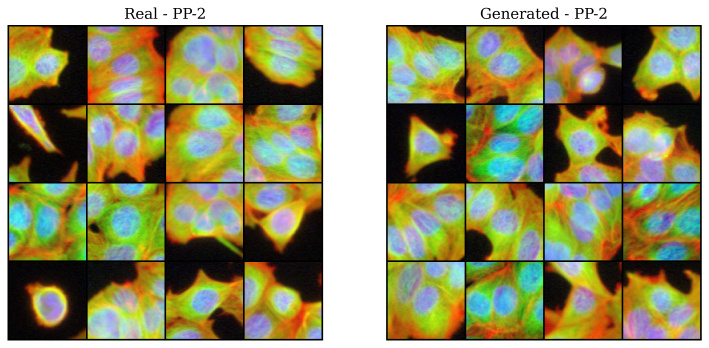}
    \includegraphics[width=0.9\linewidth]{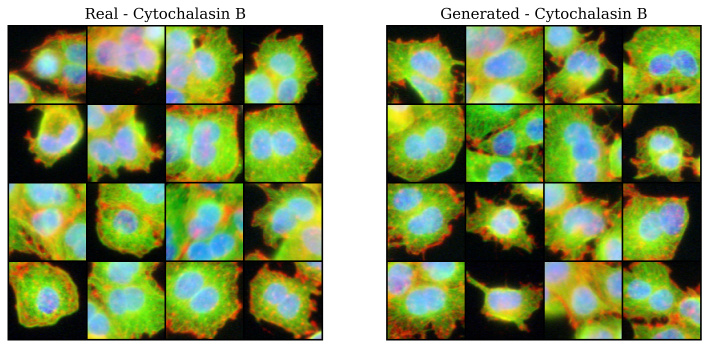}

    \caption{\textbf{Randomly sampled real and generated images for three \textit{seen} molecule perturbations on BBBC021.} Samples are generated using our model (trained from scratch on BBBC021 and finetuned with Morgan fingerprints conditioning), matching the composition of perturbations and experiments to the randomly sampled real images.}
    \label{fig:bbbc_iid_samples_extra}
\end{figure}

\clearpage
\begin{figure}[htbp]
    \centering
    \includegraphics[width=0.9\linewidth]{figures/grid_pert-AZ841.jpg}
    \includegraphics[width=0.9\linewidth]{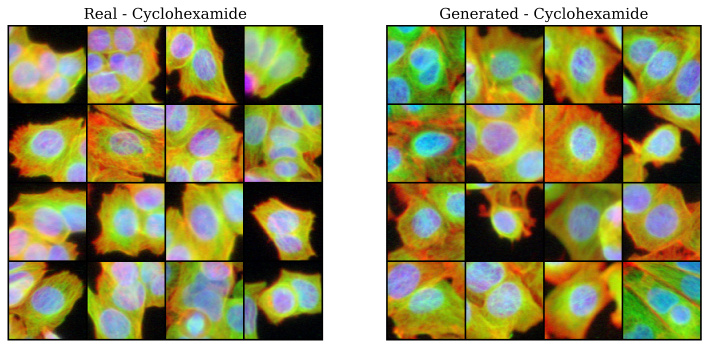}
    \includegraphics[width=0.9\linewidth]{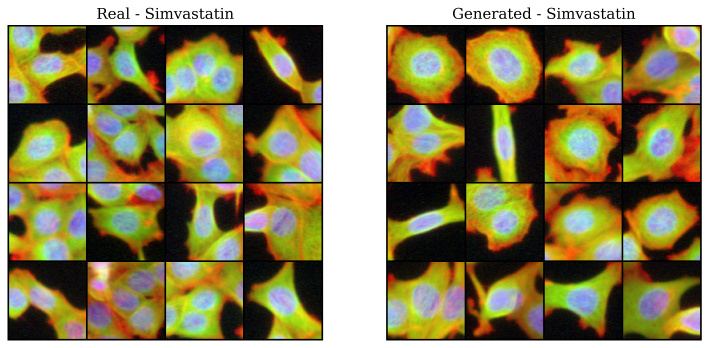}

    \caption{\textbf{Randomly sampled real and generated samples for three \textit{unseen} molecule perturbations on BBBC021.} Samples are generated using our model (trained from scratch on BBBC021 and finetuned with Morgan fingerprints conditioning), matching the composition of perturbations and experiments to the randomly sampled real images.}
    \label{fig:bbbc_ood_samples_extra}
\end{figure}
\clearpage

\section{Related work}

We provide a brief overview of relevant related work. For readers interested in a more comprehensive background on diffusion fundamentals, we recommend the monograph by \citet{lai2025principlesdiffusionmodels}. Similarly, \citet{li2025flow} provide a survey on further applications of flow matching in life sciences.

\paragraph{Diffusion and flow matching generative models}

Image-generative modelling has a large body of literature focusing on methods such as GANs \citep{NIPS2014_f033ed80} and VAEs \citep{kingma2022autoencodingvariationalbayes,pmlr-v32-rezende14}. However, in recent years, (score-based) diffusion models \citep{sohl2015deep,NEURIPS2019_3001ef25,NEURIPS2020_4c5bcfec} have come to dominate the field, generating high fidelity images \citep{dhariwal2021diffusion} with stable and scalable training. More recently, \citet{lipman2023flow}, contemporaneously with \citet{albergo2023stochastic} and \citet{liu2023flow}, proposed flow matching, providing a more general formulation of distribution-to-distribution modelling which \citet{albergo2023stochastic} prove is equivalent to diffusion with a Gaussian source distribution. Since then, improvements in architecture \citep{peebles2023scalable,karras2024analyzing} and training methods \citep{karras2022elucidating,ma2024sit} have further established diffusion and Gaussian flow matching methods as state-of-the-art in natural image generation and editing tasks \citep{xie2024sanaefficienthighresolutionimage,labs2025flux1kontextflowmatching}.

\paragraph{High-throughput phenotypic screening}

Rising costs related to the conventional approach of target-based drug discovery \citep{scannell2012diagnosing,DIMASI201620} have led many labs to invest in the alternative approach of phenotypic screening \citep{broach1996high,swinney2011were}, which involves collecting and analysing large volumes of data on the effects of biological perturbations on cells. In particular, as large-scale cell-microscopy datasets have become available (e.g. BBBC021~\citealt{ljosa2012annotated}, RxRx\{1,3\}~\citealt{sypetkowski2023rxrx1,fay2023rxrx3}, and JUMP~\citealt{Chandrasekaran2023.03.23.534023}), we have increasingly seen computer vision methods applied in drug discovery, with tasks like batch effect correction \citep{sypetkowski2023rxrx1}, microscopy representation learning \citep{kraus2024masked,kenyon-dean2025vitally}, and molecular representation learning \citep{sypetkowski2024scalability} each benefitting from scaling compute and data.

\paragraph{Biological perturbation effect prediction} One task of particular interest in science and drug discovery is predicting effects of perturbations on cells such as compounds and gene knockouts. For example, mechanism of action~\citep{10.1007/978-3-030-87589-3_58}, toxicity \citep{NYFFELER2020114876}, and activity labels \citep{doi:10.1021/acs.jcim.8b00670} may all be treated as supervised prediction tasks. More recently, we have seen interest in treating perturbation effect prediction as a generative modelling task, with the ultimate goal of modelling the distribution of cell responses in both microscopy~\citep{bourou2024phendiff,navidi2025morphodiff,palma2025predicting} and transcriptomic~\citep{Klein2025.04.11.648220,Adduri2025.06.26.661135} data. Particularly relevant to our work is CellFlux (V1/V2) \citep{zhang2025cellflux,Zhangcellfluxv2}, which applies flow matching methods to attain (prior) state-of-the-art results in generative microscopy.


\newpage

\end{document}